\documentclass[letterpaper]{article}
\usepackage{aaai}
\usepackage{times}
\usepackage{helvet}
\usepackage{courier}
\usepackage{url}
\usepackage{graphicx}
\usepackage{latexsym}
\usepackage{amsmath}
\usepackage{amssymb}
\usepackage{amsthm}
\usepackage{algorithm}
\usepackage[noend]{algpseudocode}
\usepackage{subfigure}
\newcommand{\ass}[2]{#1\rightarrow#2}
\newtheorem{thm}{Theorem}
\newtheorem{lem}{Lemma}
\newtheorem{defn}{Definition}
\title{Submodular Function Maximization for Group Elevator Scheduling}
\author{Srikumar Ramalingam\\
School of Computing, \\
University of Utah, USA\\
{\tt\small srikumar@cs.utah.com}\\
\And
{\Large Arvind U. Raghunathan \and Daniel Nikovski} \\
Mitsubishi Electric Research Labs (MERL), \\
Cambridge, MA, USA\\
{\tt\small \{raghunathan,nikovski\}@merl.com}
}

\begin{document}

\maketitle

\begin{abstract}
We propose a novel approach for group elevator scheduling by formulating it as 
the maximization of submodular function under a matroid constraint. In 
particular, we propose to model the total waiting time of passengers using a 
quadratic Boolean function. The unary and pairwise terms in the function denote 
the waiting time for single and pairwise allocation of passengers to elevators, respectively. We show that this objective function is submodular. The matroid constraints ensure that every passenger is allocated to exactly one elevator. 
We use a greedy algorithm to maximize the submodular objective function, and 
derive provable guarantees on the optimality of the solution. We tested our 
algorithm using Elevate 8, a commercial-grade elevator simulator that allows 
simulation with a wide range of elevator settings. We achieve significant 
improvement over the existing algorithms. 
\end{abstract}

\section{Introduction}
Group elevator scheduling refers to the problem of assigning passenger requests to specific elevators. We exemplify 
this using a specific example below. Consider a building with $F=8$ floors and $C=3$ elevator cars as shown in 
Figure~\ref{fg.intro_elev_schedule}. Passengers requesting elevator service press the call buttons on their 
respective floors to signal the direction of their rides.  We refer to these as \emph{hall calls}.  
For example, there are two hall calls on the floors 2 and 7 requesting service to higher and lower floors respectively. 
From the standpoint of the elevators, the direction of travel of the hall calls are known. However, the number 
of passengers behind a hall call and their destinations are typically not available in most elevator systems.  
The waiting passengers are unaware of the locations of individual cars, their moving directions, or the cars 
assigned to service them.  For example in Figure~\ref{fg.intro_elev_schedule}, the first elevator car $C1$ is at the 
floor $3$ and moving upward to service a passenger in the car going to floor $5$. This information is only available 
after the passenger enters the car and presses the particular floor button in the elevator car. We refer to such requests from passengers in the cars as \emph{car calls}. The elevator cars $C2$ and $C3$ are moving downwards and upwards to service their respective passenger requests. At each floor, we can have hall call buttons in upward and downward directions. In a building with $F$ floors we can have a maximum of $2F-2$ hall calls (there can be hall calls in only 
one direction at the lobby and the top floors) at any time. With $C$ elevator cars, we can have a maximum of 
$C^{2F-2}$ assignments of hall calls to elevators.

\begin{figure}[!t]\centering
\includegraphics[height=1.8in,width=0.50\columnwidth]{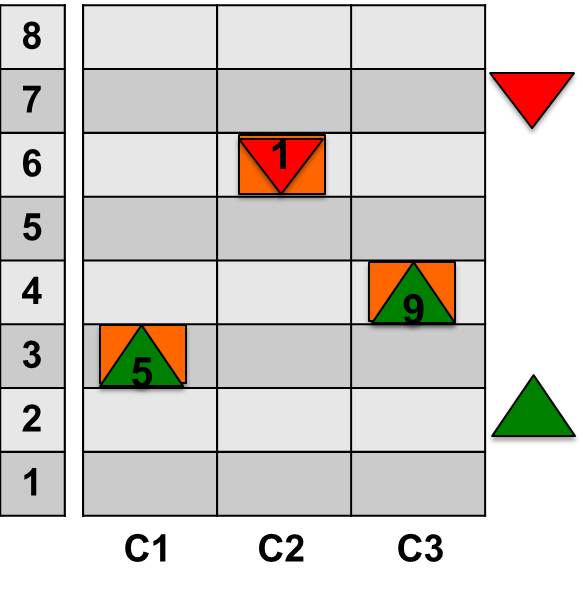}
\caption{\it A typical elevator system with 3 elevator cars and 8 floors. The three elevator cars $C1$, $C2$, 
and $C3$ are moving to serve the requests of the passengers that are in their cars. Passengers indicate the 
preference for their rides by pressing the upward or downward buttons. 
}
\label{fg.intro_elev_schedule}
\end{figure}

Given this incomplete information of the waiting passengers and the huge search space, the \emph{scheduler} in 
an elevator system has the responsibility of \emph{assigning} a particular elevator to service each of the hall calls. 
In general, the schedulers assign the hall calls to elevator cars to optimize some criteria such 
as the average waiting time, total travel time of the passengers, energy consumption, etc.  
In addition, the potential human-system interactions only serve to complicate the decision-making as even a 
gentle gesture of keeping the elevator door open for arriving passengers can render a particular assignment 
suboptimal.  As a result, the schedulers must constantly improvise and potentially change the assignments of 
hall calls to elevators as new information is revealed over time.  In fact, such reassignments are completely 
oblivious to the waiting passengers. Consequently, the problem of group elevator scheduling is a hard problem 
and continues to be researched due to its practical significance.  

In this paper, we address the problem of assigning the hall calls to elevator cars so as to 
minimize the average 
waiting time (AWT) of the passengers. The waiting time of a single passenger is the time taken from the 
moment the passenger presses the hall call button at a floor to the instant at which the passenger is picked up. 

We address the more prevalent elevator systems wherein:
\begin{itemize}
\item Hall calls do not provide information on destination floors. 
\item Number of passengers behind a hall call is not known. 
\item No prediction of future passengers is available.
\item Reassignment of hall calls is allowed.
\end{itemize}
In this setting, it is not possible to compute the {\bf exact waiting time}. Consequently, we use information on 
the current car calls and the motion of the elevator car to compute an {\textbf{estimate of the waiting times}. 
Reassignment of hall calls to elevators, wherein the scheduler can modify the previous assignments as long as the 
hall call has not been serviced, helps to mitigate the effect of the uncertainty in the hall calls. We propose an 
approximation algorithm to determine the assignments that minimize the estimated average waiting times.  

\section{Prior Work}
Most of the existing methods address one or more of the following scenarios:
\begin{itemize}
\item{Sensing and User Inputs: Algorithms that rely on additional user-interface devices that allow the users to 
enter their destination floors, or the use of sensors to detect the number of passengers waiting at each floor.}
\item{Traffic-patterns: Algorithms that are tailor-made for specific traffic patterns.}
\item{Design Criteria: Algorithms that are motivated by different design criteria such as estimated time of 
arrival, average waiting time, etc.}
\end{itemize}


\subsection{Sensing and User Inputs} In a destination control elevator system, the passengers provide the destination information to the scheduler before they board the cars. Several AI techniques~\cite{Koehler2002} have been used to address the scheduling of destination control systems, and this problem has been shown to be NP-hard even in the simplest settings~\cite{Seckinger1999}. In the case of full destination control scenario, the problem of group elevator scheduling can be formulated using planning domain definition language (PDDL)~\cite{Koehler2000,Koehler2002}.

In the presence of complete information, elevator scheduling problem has been addressed using mixed integer linear 
programming (MILP)~\cite{Ruokokoski2015,Xu2016}. Genetic algorithms have also been proposed to generate control policy 
for elevator scheduling~\cite{Dai2010}. 

In studies dating back to 1995, computer vision algorithms have been proposed to automatically count the 
number of passengers waiting at each floor, but they are not commonly used due to the cost of retro-fitting 
existing installations~\cite{Schofield1995}. The availability of complete information may become feasible in 
the case of new and high-rise buildings, but the market for such buildings is much smaller than existing 
installations. Hence, improving the performance of schedulers with incomplete information continues to be 
a very relevant and challenging problem to date. 

\subsection{Traffic-patterns}
Elevator traffic patterns can broadly be classified into three types: up-peak, down-peak, and inter-floor. In the 
up-peak case, passengers always enter the elevator system at the lobby and request upward rides. In 
the case of down-peak, the passengers request downward rides from all floors to the lobby. In the case of inter-floor the passengers request upward or downward rides between different floors, but not to or from the lobby. For up-peak traffic case, special purpose algorithms are developed based on queuing theory~\cite{Pepyne1997}. In the down-peak case, efficient algorithms have been developed using Finite Intervisit Minimization (FIM) and Empty the System Algorithm (ESA)~\cite{Bao1994}. The intensity of each traffic pattern can be represented using fuzzy variables and fuzzy logic can be 
used for the assignment of cars to hall calls~\cite{Ujihara1988,Ujihara1994}. In particular, the fuzzy rules can be used to identify the specific pattern and this can allocate a specific control algorithm for scheduling. Most of the time, the traffic patterns can be seen as the combination of these three basic patterns, e.g., at lunch time it could be 45 \% down-peak, 45 \% up-peak, and 10 \% inter-floor. One of the biggest challenges in exploiting this traffic pattern is the difficulty in accurate detection and the high frequency at which these patterns change. 

Neural networks and reinforcement learning methods have also been used for elevator scheduling problems~\cite{Crites1996,Markon1994}. In particular, Crites and Barto combined neural networks and Q-learning to demonstrate improved performance 
over FIM and ESA for one specific down-peak scenario~\cite{Crites1996}. However, it took about 60,000 hours of training time and thereby impractical for real elevator systems. On the other hand, with the recent progress in deep learning methods with the use of GPUs, this might be an interesting future direction. 

There have been algorithms that use dynamic programming to compute the exact minimization of average waiting 
time (AWT)~\cite{Nikovski2003}.  Their approach was to model it as Markov decision process (MDP). However, they assume 
that the number of passengers waiting behind a hall call is known. 

\subsection{Design Criteria}

One of the popular elevator scheduling policy is collective control, where the cars stop at the nearest call in the 
running direction~\cite{Strakosch1998}. In particular, the system computes the travel distances between hall calls 
and elevators. The scheduler assigns the hall call to the "nearest" car in terms of the distance. This strategy is 
not optimal and usually results in bunching, where several cars arrive at the same floor at the same time. 

We can also allocate hall calls by minimizing the estimated time of arrival (ETA) instead of the distance. 
To avoid the problem of bunching, we can use other approaches such as zoning or sectoring, where we partition 
the building into zones and each car can serve only one zone. In other words, we consider the same number of 
zones as the number of elevators. This strategy improves performance when the traffic is downwards towards the 
lobby. It has been shown that this strategy is also suboptimal when too many passengers arrive in the same 
zone~\cite{Barney1985,Strakosch1998}.

Relative system response (RSR) is a scheduling strategy used by Otis Elevator company to optimize a heuristic criterion based on a weighted sum of penalties and bonuses computed for every car, without any explicit relation to the actual AWT~\cite{Bittar1982}. Another criteria used by Otis that is well correlated with AWT is remaining response time (RRT), which is the time required by a car to reach the floor of the hall call~\cite{Powell1992}. Approximate methods have been used to compute the average waiting time (AWT) by predicting the number of stops for passenger pickups and the most likely floor at which the car reverses its direction~\cite{Siikonen1997,Cho1999,Barney1985}. 


\subsection{Our Contributions}
\begin{itemize}
\item To the best of our knowledge, we are not aware of any group elevator scheduling algorithm that is based on 
submodularity. In this work, we show that the problem of group elevator scheduling can be formulated as the maximization 
of submodular functions under a matroid constraint. Submodularity is one of the key concepts in machine learning, computer vision, economics, and game theory. Recently it has found application in such diverse domains as: sensor placement~\cite{Guestrin2008}, outbreak detection~\cite{Leskovec2007}, word alignment~\cite{Lin2011}, clustering~\cite{liu2013}, viral marketing~\cite{Kempe2003}, and finding diverse subsets in structured item sets~\cite{Prasad2014}.

\item Our formulation enables the use of a simple greedy algorithm with guarantees on the optimality of the 
solution. 
\item We use as objective a quadratic Boolean function that approximates the average waiting times of passengers.
\item Existing efforts in the literature demonstrate their advantages using customized simulation data, which makes it harder to evaluate and understand their usefulness. We evaluate using Elevate 8 and show consistent and significant improvement over the implemented standard scheduling algorithms in Elevate 8 over a wide variety of elevator settings. 
\end{itemize}

\section{Background}
In this section, we briefly define useful entities such as set functions, submodularity, and matroids. 

\noindent
{\bf Notations:}
Let $\mathbb B$ denote the Boolean set $\{0,1\}$ and $\mathbb R$ the set of reals. We use ${\bf x}$ 
to denote vectors. 

\begin{defn} A set function $F:2^{E} \rightarrow \mathbb{R}$, where $E$ is a finite set, maps a set to a real number. A set function can also be seen as a pseudo-Boolean function~\cite{BorosH02} that takes a Boolean vector as argument and returns a real number.
\end{defn}

\begin{defn} A set function $F:2^{E} \rightarrow \mathbb{R}$ is submodular if for all $A,B \subseteq E$ with $B \subseteq A$ and $e \in E \backslash A$, we have: 
\begin{equation}
F(A \cup \{e\}) - F(A) \le F(B \cup \{e\}) - F(B). \label{eqn.dimreturns}
\end{equation}
\label{def.submodularity}
\end{defn}

This property is also referred to as diminishing return since the gain from adding an element to a small set $B$ is never smaller than adding it to a superset $A \supset B$~\cite{George_Nemhauser_MP1978}.

\begin{defn}
A set function $F$ is monotonically increasing if for all $A,B \subseteq E$ and $B \subseteq A$, we have:
\begin{equation}
F(B) \le F(A).
\end{equation}
\end{defn}

\begin{defn}
A matroid is an ordered pair $M=(E,\mathcal{I})$ consisting of a finite set $E$ and a 
collection $\mathcal{I}$ of subsets of $E$ satisfying the following three conditions:
\begin{enumerate}
\item $\emptyset \in \mathcal{I}$.
\item If $I\in \mathcal{I}$ and $I^{\prime}\subseteq I$, then $I^{\prime} \in \mathcal{I}$.
\item If $I_1$ and $I_2$ are in $\mathcal{I}$ and $|I_1|<|I_2|$, then there is an element 
$e \in I_2-I_1$ such that $I_1\cup \{e\} \in \mathcal{I}$.
\end{enumerate}
\end{defn}

The members of $\mathcal{I}$ are the independent sets of $M$. Note that there exist several other definitions for matroids that are equivalent. For more details, one can refer to~\cite{James_Oxley_1992}.

\begin{defn}
Let $E_1,\ldots,E_n$ be the partition of the finite set $E$ and let $k_1,\ldots,k_n$ be positive integers. The ordered pair $M=(E,\mathcal{I})$ is a partition matroid if:
\begin{equation}
\mathcal{I} = \{I: I \subseteq E,~~~~~|I \cap E_i| \le k_i, 1 \le i \le n\}.
\end{equation} 
\end{defn}

\section{Problem Formulation}
In this section, we formulate group elevator scheduling as a quadratic Boolean optimization problem. 
The qualifier ``quadratic'' is due to our choice of model for the waiting times of passengers.  

To set the stage, suppose there are $N$ hall calls that need to be assigned to elevators.  Note that these 
$N$ hall calls include previously assigned calls that are considered for reassignment and new hall calls 
that have just requested service. We do not allow reassignment when the assigned car is already close to 
servicing the assigned hall call.

The decision variables in the problem are denoted by $x^c_i$. The variable $x^c_i$ is a Boolean variable to 
denote the assignment of the hall call $i$ to elevator car $c$:
\begin{equation}
x^c_i =
\begin{cases}
    1 & \mbox{if the hall call $i$ is assigned to car $c$,}\\
    0 & \mbox{otherwise.}
  \end{cases}
\end{equation}
Let ${\bf x} \in \{0,1\}^{N\cdot C}$ denote the vector storing all the 
assignments $\{x^1_1,x^1_2,\dots,x^1_N,\dots,x^C_N\}$. 
In any \emph{feasible assignment} of hall calls to elevator cars, each hall call must be assigned to exactly 
one elevator car.  This can be imposed through the constraint,
\begin{equation}
\sum_{c=1}^C x^c_i = 1,~~\forall i \in \{1,...,N\}. \label{feasassign}
\end{equation}
The number of feasible assignments is $C^N$.  If we denote the waiting time associated with a feasible 
assignment ${\bf x}$ as $w^{\text{true}}({\bf x})$ then the group elevator scheduling problem can be posed as 
minimizing $w^{\text{true}}({\bf x})$ over all feasible assignments ${\bf x}$.  Clearly, computing the terms in the 
objective of such an optimization problem is an onerous task and is infeasible for real elevator systems 
in which decisions are typically made every second or fractions of a second.

To alleviate this bottleneck we propose the following approximation for the total waiting time to serve the 
hall calls,
\begin{equation}
g({\bf x})=\sum_{i=1}^{N}\sum_{c=1}^{C}w_i^c x_i^c + 
\sum_{i=1}^{N}\sum_{j=i+1}^{N}\sum_{c=1}^{C}w_{ij}^c x_i^c x_j^c \label{obj}
\end{equation}
where $w_i^c$, $w_{ij}^c$ are computed based on the current state of the elevator car $c$ as explained next.  
The term $w_i^c$ is the waiting time for car $c$ to pick up the passenger(s) for hall call $i$ given the 
current set of car calls for the car $c$.  Consider the assignment of the hall calls to car C1 in  Figure~\ref{fg.intro_elev_schedule}.   We first consider the assignment of hall call on floor $3$ to car C1 as shown 
in Figure~\ref{figwa}.  Since C1 already has a car call to floor 5 the waiting time for the hall call can be computed as
\[
w^c_i = t_{3 \rightarrow 5} + t_{door} + t_{5 \rightarrow 2}
\]
where $t_{i \rightarrow j}$ represents the time to travel from floor $i$ to floor $j$ and $t_{door}$ includes the sum of door 
opening, dwell and closing times.  Car C1 moves from a state of rest to state of rest when it travels 
from floor 5 to floor 2.  On other hand, for the journey from floor 3 to 5 the car may already be in motion or at rest.  
Hence, the time required needs to be computed according to the kinematic state of the elevator car.  Such 
kinemattic formulas have been derived in~\cite{peters93} as a function of the velocity, acceleration and jerk of the car 
and the maximum limits for the same.   Note that the other hall calls are completely ignored in performing this computation.  We refer to this as \emph{unary} term since the influence of other hall calls are completely 
ignored in the computation.  In a similar manner, we can compute the waiting time $w^c_j$ for picking the passenger 
requesting hall call from floor 7 (refer Figure~\ref{figwb}).  

The term $w^c_{ij}$ represents the excess over the $w^c_i + w^c_j$ that is incurred when both hall calls $i$ and 
$j$ are assigned to the same car $c$.  We refer to this term as \emph{pairwise} term. 
In other words, $(w^c_i + w^c_j + w^c_{ij})$ is the total waiting time to 
pick up passengers for hall calls on floors $i$ and $j$ given the current set of car calls for car $c$ 
and can be obtained as
\[\begin{aligned}
w^c_{ij} =&\, t_{3 \rightarrow 5} + t_{door} + t_{5 \rightarrow 7} + t_{door} + t_{7 \rightarrow f} + t_{door} \\
&\,+ t_{f \rightarrow 2} - w^c_i - w^c_j
\end{aligned}\]  
where $f$ represents the \emph{unknown} destination floor of the passenger on floor 7.  In the computation of 
$w^c_{ij}$ we have chosen to serve the hall call at floor 7 before serving the hall call at floor 2.  This choice is  
dictated by typical elevator movement rules.  Once a car is empty, 
\begin{itemize}
\item The car maintains its upward (downward) direction of motion if there are up (down) hall calls at floors above (below) the current floor of the elevator car.  
\item If not, the car moves to the highest (lowest) floor at which downward (upward) hall call exists.  
\item If not, then the car moves downward (upward) until the lowest (highest) floor with hall call in the 
upward (downward) direction exists.  
\end{itemize}
Unlike in the computation of the unary terms, notice that we require knowledge of the destination floor $f$ 
for the first hall call that is serviced.  Since this information is not available we instead compute the pairwise 
term as an expectation over all possible destination floors for that hall call,
\[\begin{aligned}
w^c_{ij} =& \sum\limits_{f \in \{1,\ldots,6\}} \omega_f \cdot \left(\begin{aligned} 
t_{3 \rightarrow 5} + t_{door} + t_{5 \rightarrow 7} + t_{door} \\
+ t_{7 \rightarrow f} + t_{door} + t_{f \rightarrow 2} 
\end{aligned}\right) \\
& - w^c_i - w^c_j
\end{aligned}\]  
where $\omega_f$ is the probability of floor $f$ being the destination.  If there exists additional information 
on the probabilities these can be easily incorporated in the computation for the pairwise term.  However, in the 
absence of such information we assume that destination floors are equally likely. Observe that the set of destination floors is only $\{1,\ldots,6\}$ since the hall call requests downward service from floor 7. As in the computation of the unary terms, the remaining hall calls are ignored in performing this computation. We can state the following result on the pairwise term.

\begin{lem}\label{lem:nonnegativewij}
$w^c_{ij} \geq 0$.
\end{lem}
\begin{proof}
Suppose two hall calls $i,j$ are assigned to the same car and that hall call $i$ is serviced first.  
The waiting time for hall call $i$ is exactly $w^c_i$.  If the second hall call is also on the same floor then 
the total waiting time is $w^c_i+w^c_j$.  If not the waiting time for hall call $j$ is greater than $w^c_j$ since the 
car makes an intermediate stop and there is also time associated with door operation to pick up the passenger for hall call $i$.  Hence, the pairwise term is always nonnegative.
\end{proof}

Clearly, $g({\bf x})$ is \emph{exact} when no more than 2 hall calls are assigned to each elevator car.  
For higher number of assignments to a car this is only a \emph{proxy} for the actual waiting times of the 
passengers.  However, our choice of this quadratic form for $g(\bf x)$ is motivated by:
\begin{itemize}
\item  \emph{nice} properties that (we prove in the next section) that allow us to derive simple algorithms with 
provable guarantees
\item reduced computational effort in computing the objective - $O(N\cdot C + N(N-1)\cdot C)$.
\end{itemize}

\begin{figure}[h]
\centering
\subfigure[$w_i^c$]{\includegraphics[scale=0.35]{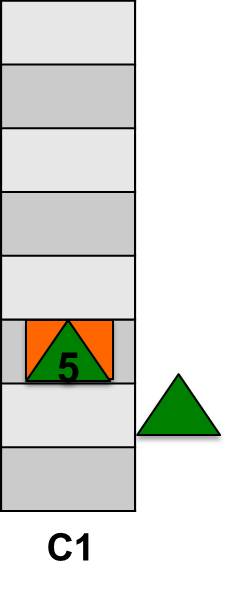}\label{figwa}}\qquad
\subfigure[$w_j^c$]{\includegraphics[scale=0.35]{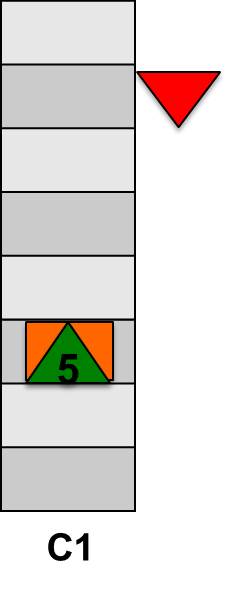}\label{figwb}}\qquad
\subfigure[$w_{ij}^c$]{\includegraphics[scale=0.35]{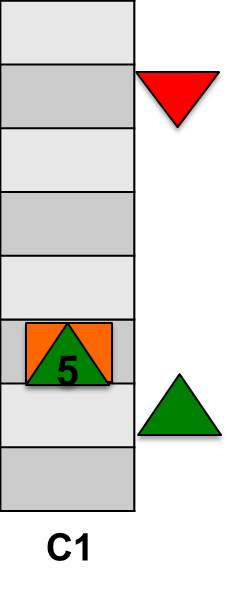}\label{figwc}}
\caption{Computation of unary and pairwise terms.}\label{figw}
\end{figure}

Let ${\bf w}$ denote the vector storing all the weights $\{w^1_1,\dots,w_1^C,w^1_2,\dots,w_N^C,w_{12}^1,\dots,w_{{N-1}N}^C\}$.

We write the group elevator scheduling problem to minimize the waiting time as the following quadratic 
Boolean optimization problem:
\begin{equation}
\begin{aligned}
\min_{{\bf x}} &\,\, g(\bf x) \\
\text{s.t} &\,\, \sum_{c=1}^C x^c_i = 1,~~\forall i \in \{1,...,N\}, \\
&\,\, {\bf x} \in \{0,1\}^{N\cdot C}.
\end{aligned}
\label{qbo}
\end{equation}

\section{Submodular Maximization \& Matroid}
In the section we show that the quadratic Boolean optimization in Equation~\eqref{qbo} can be posed as a 
submodular maximization problem over a matroid. To begin with, we pose the problem in Equation~\eqref{qbo} as a maximization using 
negations in the objective function:
\begin{equation}
\begin{aligned}
\max_{{\bf x}} &\,\, -g(\bf x) \\
\text{s.t} &\,\, \sum_{c=1}^C x^c_i = 1,~~\forall i \in \{1,...,N\}, \\
&\,\, {\bf x} \in \{0,1\}^{N\cdot C}.
\end{aligned}
\label{eq.qbo}
\end{equation}
Some related notation on set functions that will allow us to formulate the problem over sets and help us to introduce 
the matroid constraints. 

Let us consider a finite set $E$ comprising the assignments:
\begin{equation*}
E = \{\ass{1}{1},\ldots,\ass{N}{1},\ldots,\ass{1}{C},\ldots,\ass{N}{C}\}
\end{equation*}
where $\ass{i}{c}$ denotes that hall call $i$ is assigned to car $c$.  
The objective function $g(\bf x)$ can be equivalently defined over the subsets of $E$ as follows. 
Let $h:2^{E} \rightarrow \mathbb{R}$ be a set function defined for for any set $A \subset E$ as:
\begin{align}
& h(A) = -g({\bf x}), \label{eq.submodular_objfunction} \\
\text{ where }		& x^c_i = 
	\begin{cases}
    1 & \mbox{$\ass{i}{c} \in A$,}\\
    0 & \mbox{otherwise}
  \end{cases}, \forall \left\{\begin{aligned} c \in \{1,\dots,C\}, \\ i \in \{1,\dots,N\} \end{aligned}\right. .
	\label{eq.uniquenessConstraint}
\end{align}.

\begin{lem}
The function $h(A)$ is submodular.
\label{lem.sumodularity1}
\end{lem}
\begin{proof}
Consider two sets $A$ and $B$ where $B \subseteq A$ and $A,B \subseteq E$. Let 
$\ass{i}{c} \in E \backslash A$. We can observe that:
\begin{equation}
\begin{aligned}
h(A \cup \{\ass{i}{c}\}) - h(A) =  - w_i^c - \sum_{(j,c) : \ass{j}{c} \in A} w_{ij}^c  \\
h(B \cup \{\ass{i}{c}\}) - h(B) =  - w_i^c - \sum_{(j,c) : \ass{j}{c} \in B} w_{ij}^c \nonumber 
\end{aligned}\label{eqn.incrchange}
\end{equation}
Since unary and pairwise terms are non-negative (refer Lemma~\ref{lem:nonnegativewij}) and $B \subseteq A$, 
\begin{equation*}
h(A \cup \{\ass{i}{c}\}) - h(A) \le h(B \cup \{\ass{i}{c}\}) - h(B)
\end{equation*}
which proves the claim.
\end{proof}


Now we consider the constraint in Equation~\eqref{feasassign} modeling the assignment of a hall call to exactly one elevator car.  We show that this can be formulated using partition matroid constraint. In order to do that, let us consider $N$ disjoint subsets of the ground set $E$ as shown below:
\begin{eqnarray*}
E_1 & = &  \{\ass{1}{1},\ldots,\ass{1}{C}\} \\
&\vdots & \\
E_N & = &  \{\ass{N}{1},\ldots,\ass{N}{C}\}.
\end{eqnarray*}
Let $M=(E,\mathcal{I})$ denote a partition matroid such that each independent set shares no more than one element with each of the disjoint subsets as shown below:
\begin{equation}
|I \cap E_i | \le 1,~~~~~I \in \mathcal{I},\forall i \in \{1,\dots,N\}
\end{equation}

Let us consider the following submodular function maximization under a partition matroid constraint:
\begin{equation}
\begin{aligned}
\max_{A \subseteq E} &\,\, h(A) \\
\text{s.t.} &\,\, A \in \mathcal{I}, M = \{E,\mathcal{I}\}.
\end{aligned}
\tag{SFM-1}\label{eq.sfm1}
\end{equation}
Note that the above optimization is different from the maximization problem given in Equation~\ref{eq.qbo} because the constraints are different. The partition matroid only enforces that each hall call is assigned to no more than one elevator car (i.e., $|A \cap E_i| \leq 1$). However, it does not enforce that each hall call is assigned to at least one elevator car (i.e., $|A \cap E_i| \geq 1$ is not enforced). 
%
In order to do that, we add a penalty term such that an assignment that violates equation~\eqref{feasassign} (i.e. $|A \cap E_i| < 1$) has a lower objective than any feasible assignment. This will ensure that each hall call is assigned to at least one elevator car. We consider the following additional term $h_1(A)$ in the objective function:
\begin{equation}
\begin{aligned}
& h_1(A) = - \sum\limits_{i=1}^N p_i \cdot (C - |A\cap E_i|) \\
\text{where, } &p_i = \max_{c \in \{1,\dots,C\}} \left( w^c_i + \sum_{j = 1, j \neq i}^N w^c_{ij} \right).
\end{aligned}\label{eqn.defh1}
\end{equation}
The term $h_1(A)$ reduces the objective function by $C\sum_{i=1}^N p_i$ if there is no assignment, i.e., 
$A = \emptyset $.  Since $p_i \ge 0$, we increase the objective function for every assignment of hall call to 
a car ($\ass{i}{c}$) included in the set $A$.  Further, the function $h_1(A)$ satisfies the following property 
for any $\ass{i}{c} \notin A$,
\begin{equation}
h_1(A \cup \{\ass{i}{c}\}) - h_1(A) = p_i. \label{eqn.proph1}
\end{equation}
The following lemma shows that the combined objective function is submodular. 
\begin{lem}
The function $h(A) + h_1(A)$ is submodular.
\label{lem.sumodularity2}
\end{lem}
\begin{proof}
Consider two sets $A, B \subset E$ with $B \subset A$.  Let $\ass{i}{c} \in E\setminus A$.  Then, by 
Equation~\eqref{eqn.proph1} we have that
\[
h_1(A \cup\{\ass{i}{c}\}) - h_1(A) = h_1(B \cup \{\ass{i}{c}\}) - h_1(B) = p_i.
\]
In other words, $h_1(A)$ satisfies the condition for submodularity (Equation~\eqref{eqn.dimreturns}) in 
Definition~\ref{def.submodularity} as an equality.  
The sum of two submodular functions is submodular and thus $h(A) + h_1(A)$ is submodular. 
\end{proof}

We consider the following submodular function maximization problem that is equivalent to the quadratic 
Boolean optimization problem in Equation~\ref{eq.qbo}.
\begin{equation}
\begin{aligned}
\max_{A \subseteq E} &\,\, h(A) + h_1(A) \\
\text{s.t.} &\,\, A \in \mathcal{I}, M = \{E,\mathcal{I}\}. 
\end{aligned}
\tag{SFM-2}\label{eq.sfm2}
\end{equation}
In this work, we use a greedy algorithm to solve the maximization of the submodular function under the matroid constraint. 
The use of greedy algorithm is motivated by the following theorem:

\begin{thm}~\cite{George_Nemhauser_MP1978}\label{thm1}
For maximizing monotonically non-decreasing submodular functions under a matroid constraint, the optimality of the 
greedy algorithm is characterized by the following equation:
\begin{equation}
f(A_{greedy}) \ge \frac{1}{2} f(A_{OPT}),
\label{eq.greedy_bound}
\end{equation}
where $f(\emptyset) = 0$. 
\label{th.greedy_theorem}
\end{thm}

In the submodular optimization problem given in Equation~\eqref{eq.sfm2}, the objective function is not equal to 
zero when $A = \emptyset$. The optimality bound given in Theorem~\ref{th.greedy_theorem} applies for 
submodular function maximization where the function is equal to zero when the solution is an empty set. 
From Equation~\eqref{eqn.defh1}, $h(\emptyset) + h_1(\emptyset) = -C\sum_{i=1}^N p_i$.  
By adding a constant term ($C\sum_{i=1}^N p_i$) we can ensure that the requirement of Theorem~\ref{thm1} can 
be satisfied.  Further, the addition of a constant to a submodular function ($h(A) + h_1(A)$), also ensures 
submodularity of the resulting function $h(A) + h_1(A) + C\sum_{i=1}^N p_i$ (refer Definition~\ref{def.submodularity}). 
We consider the following optimization problem which is equivalent to SFM-2:
\begin{equation}
\begin{aligned}
\max_{A \subseteq E} &\,\, h(A) + h_1(A) + C\sum_{i=1}^N p_i \\
\text{s.t.} &\,\, A \in \mathcal{I}, M = \{E,\mathcal{I}\}.
\end{aligned}\tag{SFM-3}\label{eq.sfm3}
\end{equation}
In order to apply Theorem~\ref{th.greedy_theorem} to the above optimization problem, we also have to show that the objective function is monotonically non-decreasing.  
\begin{lem}
The function $h(B) + h_1(B) + C\sum_{i=1}^N p_i$ is monotonically non-decreasing. 
\label{lem.monotonic}
\end{lem}
\begin{proof}
The constant term $C\sum_{i=1}^N p_i$ does not affect the monotonicity of a function. 
Thus we need to only show that $h(B) + h_1(B)$ is 
monotonically non-decreasing for any $B \subset E$. Suppose that $\ass{i}{c} \notin B$ is added to the set $B$.  
From Equations~\eqref{eqn.incrchange} and~\eqref{eqn.proph1},  
\begin{align}
&\, h(B \cup \{\ass{i}{c}\}) + h_1(B  \cup \{\ass{i}{c}\}) - 
h(B) - h_1(B) \nonumber \\
=&\,- w_i^c  - \sum_{(j,c) : \ass{j}{c} \in B} w^c_{ij} + p_i \nonumber \\		
=&\, - w_i^c  - \sum_{(j,c) : \ass{j}{c} \in B} w^c_{ij} + 
\max_{c \in \{1,\dots,C\}} \left( w^c_i + \sum_{j = 1, j \neq i}^N w^c_{ij} \right).
\nonumber
\end{align}
Since the unary and pairwise terms are non-negative (refer Lemma~\ref{lem:nonnegativewij}), 
it can be readily seen that:
\begin{equation*}
\begin{aligned}
&\, \max_{c \in \{1,\dots,C\}} \left( w^c_i + \sum_{j = 1, j \neq i}^N w^c_{ij} \right) \\
\ge&\, w_i^c x_i^c + \sum_{(j,c) : \ass{j}{c} \in B} w^c_{ij} \ge 0. 
\end{aligned} 
\end{equation*}
Hence
\begin{equation}
h(B \cup \{\ass{i}{c}\}) + h_1(B  \cup \{\ass{i}{c}\}) \ge h(B) + h_1(B). \label{eqn.nondecr}
\end{equation}
Further any set $A$ with $B \subseteq A \subseteq E$ can be obtained by incrementally adding to set $B$ the 
elements in $A \setminus B$.  Thus, iterative application of Equation~\eqref{eqn.nondecr}  yields
\begin{equation*}
h(B) + h_1(B) \le h(A) + h_1(A) 
\end{equation*}
proving the claim.
\end{proof}

For the sake of completeness, we list the steps of the greedy algorithm used in maximization of 
submodular function under a matroid constraint~\cite{George_Nemhauser_MP1978}. Let the 
objective function in Equation~\eqref{eq.sfm3} be denoted by 
$f(A) = h(A) + h_1(A) + C\sum_{i=1}^N p_i$. 
\newline

\noindent
{\bf Greedy Algorithm:}
\begin{enumerate}
\item Initialize $S = \emptyset$.
\item Let $s = \arg\max_{s' \in E} f(S \cup \{s'\}) - f(S)$ such that $S \cup \{s'\} \in \mathcal{I}$.
\item If $s \ne \emptyset$ then $S = S \cup \{s\}$ and go to step 2. 
\item $S$ is the required solution. 
\end{enumerate}

Note that the objective function used in Equation~\eqref{eq.sfm3} is not just the average waiting time. 
We also add a penalty $h_1(A)$ and a constant term $C\sum_{i=1}^N p_i$. Thus the actual bound on the 
optimality using greedy algorithm is
\begin{equation}
h(A_{greedy}) + \sum_{i=1}^N p_i \ge \frac{1}{2} \left( h(A_{OPT}) + \sum_{i=1}^N p_i \right) \label{eqn.optbnd}
\end{equation}
Note the penalty term evaluates to $h_1(A) = -(C-1)\sum_{i=1}^Np_i$ for valid assignments where every hall call is assigned to 
one elevator car. Taking into account the constant term $C\sum_{i=1}^n p_i$ yields the offset in 
Equation~\eqref{eqn.optbnd}. 

We can also consider higher order terms to impose penalty on non-balanced assignment of passengers to elevator cars, i.e., assigning most of the passengers to single elevator cars. For example, we can impose a penalty whenever three passengers are assigned to a single elevator car. This penalty can be imposed using a third degree term, i.e., $ x_i^c x_j^c x_l^c$.  By adding a higher order term of order $k$ with negative coefficients to the function $h(A)$ the resulting function is known to be submodular~\cite{BorosH02,kolPAMI04}.


\section{Experiments}


We implemented our submodular maximization based greedy algorithm within the elevator simulator 
{\bf Elevate 8}\footnote{https://www.peters-research.com/index.php/8-elevate/58-elevate-8}. 
This is a commercial-grade simulator that allows the selection of 
the number of floors, the number of elevator cars, the speed/acceleration of cars, height of the 
floors, etc. It allows the user to choose different traffic patterns such as up-peak, down-peak, and 
inter-floor. In particular, the simulator provides several industry-standard methods such as group 
collective control, estimated time of arrival (ETA), destination control, etc. 

In addition to the details described earlier, we outline a few other enhancements to our algorithm:
\begin{itemize}
\item We considered only non-destination control scenarios. Since we do not know the destination floors, 
we consider all possible destinations and use their average to compute the delay. 
\item For the elevator cars that are close to capacity, we use a penalty term to avoid assigning  additional 
passengers. This is achieved by using a high unary cost for assigning additional passengers 
to these elevator cars. This ensures that these calls are assigned by the greedy algorithm in the later 
stage. After the greedy assignment, we remove these assignments from the respective elevators.
\item Door status determines the amount of time that elapses before the elevator can move away from a floor. 
This is important for correct assignment in low traffic conditions.  The simulator provides information on the 
door status and we appropriately include the additional time due to door operation. 
\item We give a bonus to hall calls assigned to a car with existing car calls for the hall call floor.  
The bonus is provided in the form of reduction of the unary term associated with the particular 
assignment.  The reduction is specified as,
\begin{equation}
w_i^c = w_i^c - \min (0.20 w_i^c, 10). \label{eq.bonus}
\end{equation}
\item We penalize assignments of too many hall calls to the same elevator car. This is achieved using higher order terms. 
For example, by adding the term $w_{i_1 i_2 ... i_k}^cx_1^c x_2^c ... x_k^c$ in the objective function we increase the waiting time 
by $w_{i_1 i_2 ... i_k}^c$ if $k$ passengers board the same elevator car $c$. 
\end{itemize}
It is important to note that the above changes to the cost function still preserve submodularity. 

We used the following experimental setup to evaluate the different scheduling algorithms.  
We studied:
\begin{itemize}
\item three different buildings with 8, 10, and 12 floors 
\item for each building we consider a 1-hour period of inter-floor traffic scenario with 5 different arrival rates 
specified as percentage of passengers in 5 minutes - 10\%, 15\%, 20\%, 25\% and 30\%
\item for each building we considered 2,~3,~4,~5, and 6 elevator cars. 
\end{itemize}
For each specific building, arrival rate of passengers and fixed number of elevator cars, we measure 
the average waiting time achieved by the scheduling algorithm as the average over 10 different 
instances of the traffic scenario.   

We compare our methods with 
group collective control and estimated time of arrival (ETA) methods. The said methods are known to work well under different traffic conditions and elevator settings. The implementation in Elevate 8 employs several heuristics for improved performance: accounting for future demands, priority for coincident calls, not stopping an elevator car when it is full, intelligent decisions about keeping the door open before getting the hall call, detecting up-peak or down-peak scenarios automatically, etc. However, such heuristics are not included in an explicit objective function and it is hard to determine the impact of some of these improvements. In all our experiments, we consistently outperform these methods in a wide variety of elevator settings. 

We begin by describing an ablation study on our algorithm. Figure~\ref{fg.unary_pairwise_8floor} plots the 
percentage savings in average waiting time when using both unary and pairwise terms in the objective function 
as opposed to using only the unary term in the objective.   For the latter case we simply set $w^c_{ij} = 0$.  
The plot provides the savings at different number of cars and arrival rates of  traffic.  In this study, we did not 
include the bonus specified in Equation~\eqref{eq.bonus}.  Figure~\ref{fg.unary_pairwise_8floor} clearly shows that 
pairwise term is critical to obtaining higher savings in average waiting time across all buildings. The average 
reduction over all scenarios is about $10.9$ \%.

\begin{figure}\centering
\includegraphics[width=\columnwidth,height=5.2cm]{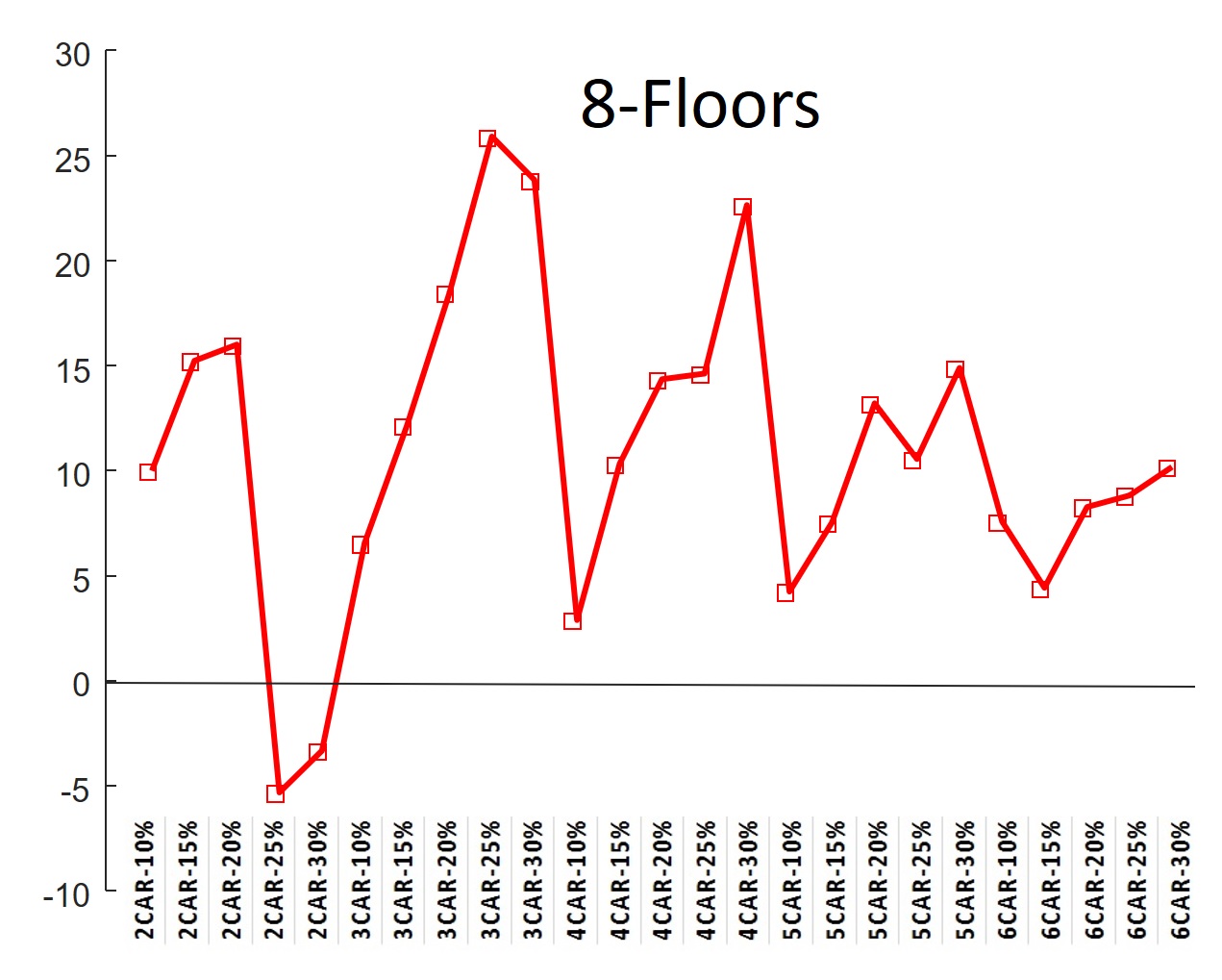}
\caption{\it We show the percentage decrease in the average waiting time using pairwise terms.}
\label{fg.unary_pairwise_8floor}
\end{figure}
\begin{figure}\centering
\includegraphics[width=\columnwidth,height=5.2cm]{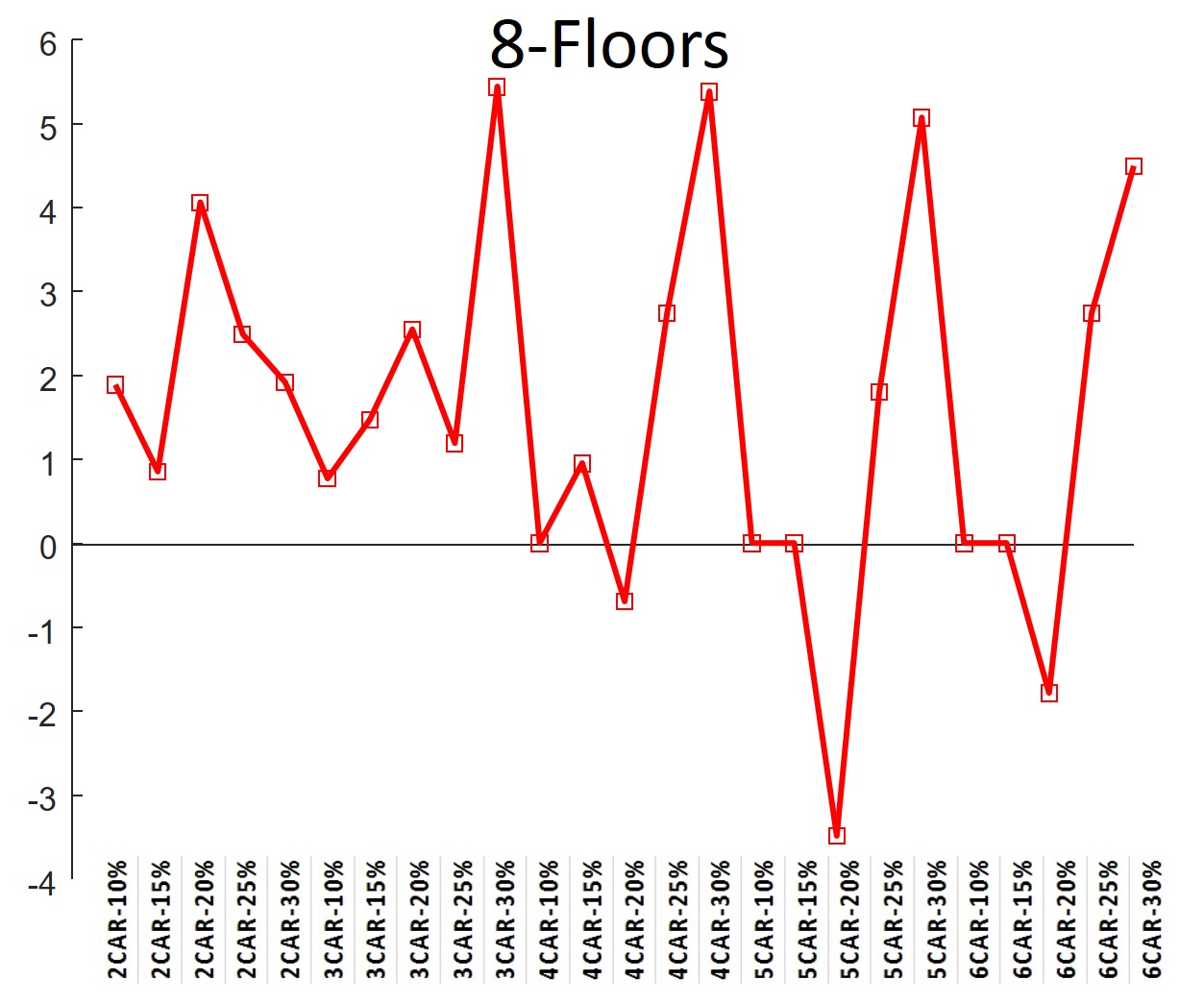}
\caption{\it We show the percentage decrease in the average waiting time using coincident call bonus.}
\label{fg.coincident_call_benefit}
\end{figure}

Figure~\ref{fg.coincident_call_benefit} considers the effect of including the bonus in Equation~\eqref{eq.bonus}.  
The plot shows the reduction in average waiting times for the objective using unary and pairwise terms with 
the bonus from Equation~\eqref{eq.bonus} over the case where the bonus is not included. The average 
reduction over all scenarios is about $1.6$ \%.

We now compare our scheduling algorithms against the ones in Elevate.  
Figure~\ref{fg.group_collective_verses_submodular} plots the reduction in average waiting time over group collective control for different arrival rates and fixed elevators for three different buildings.  
For the case of 8-floors we obtain an average reduction of 8.6 \%, 5.3 \% for 10-floors and 3.9 \% for 
12-floors.


Figure~\ref{fg.eta_verses_submodular} plots the reduction in average waiting time over ETA for different arrival rates and fixed elevators for three different buildings.  
For the case of 8-floors we obtain an average reduction of 4.4 \%, 3.9 \% for 10-floors and 4.2 \% for 12-floors. To illustrate the utility of higher-order terms, we use a small penalty for discouraging the assignment of 4 or 5 passengers to the same car and obtain further reduction in waiting time. For the case of 8-floors we obtain an average reduction of 4.6 \% over ETA as shown in Figure~\ref{fg.horder_eta_verses_submodular}.


\begin{figure}\centering
\includegraphics[width=\columnwidth,height=5.2cm]{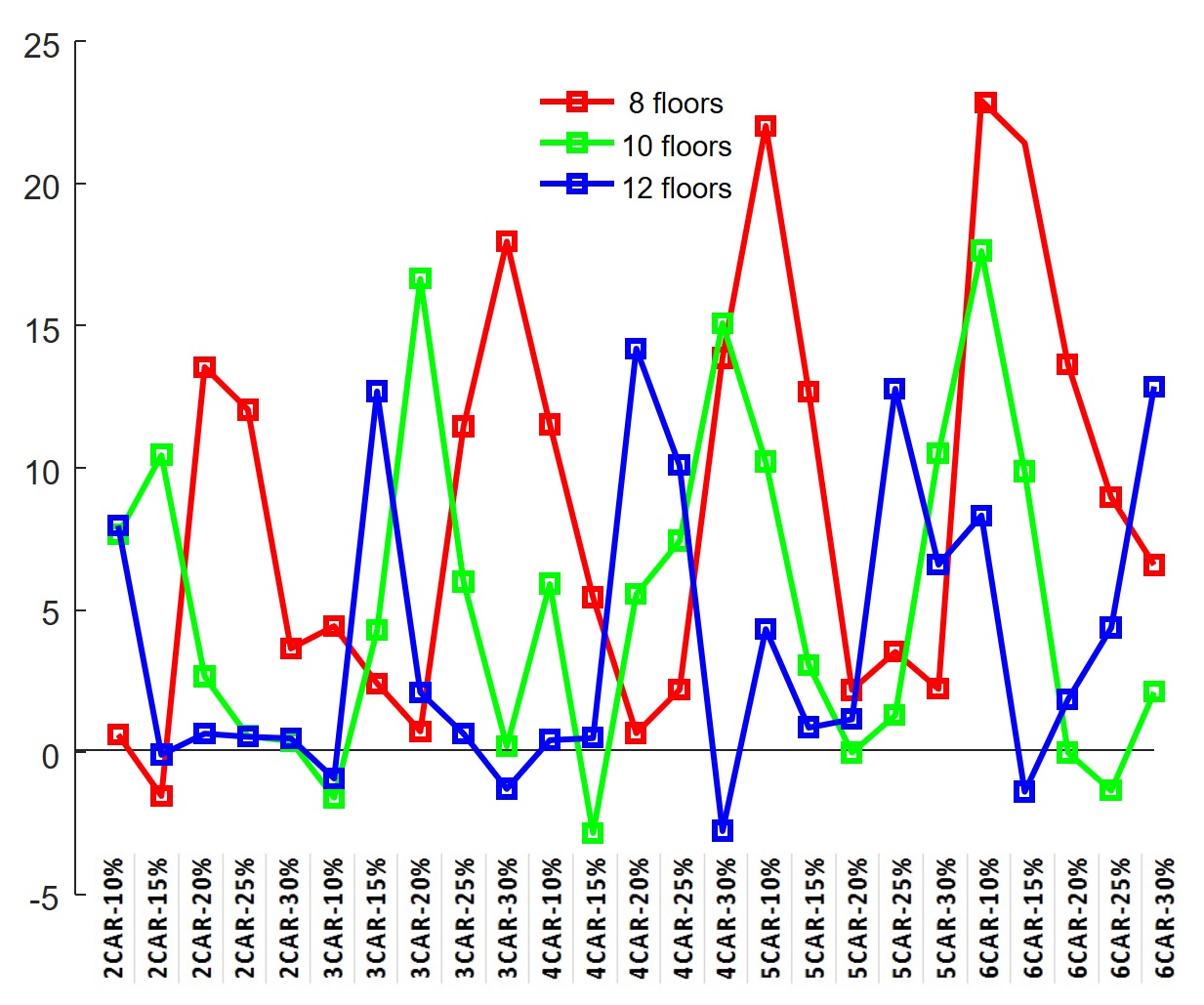}
\caption{\it We show the percentage decrease in the average waiting time of our submodular approach 
with respect to the group collective control.}
\label{fg.group_collective_verses_submodular}
\end{figure}
\begin{figure}[t]\centering
\includegraphics[width=\columnwidth,height=5.2cm]{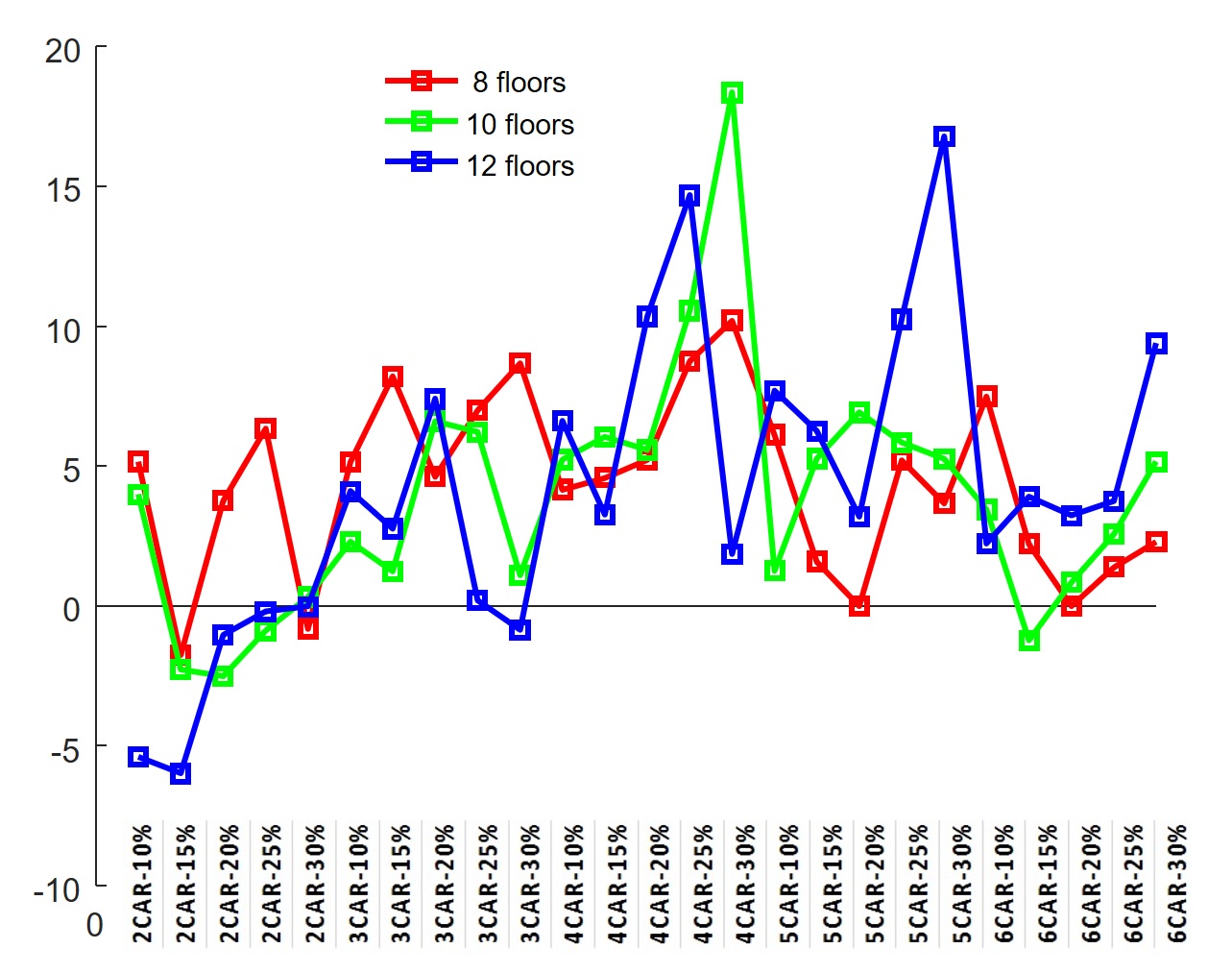}
\caption{\it We show the percentage decrease in the average waiting time of our submodular approach 
with respect to ETA.}
\label{fg.eta_verses_submodular}
\end{figure}
\begin{figure}[t]\centering
\includegraphics[width=\columnwidth,height=5.2cm]{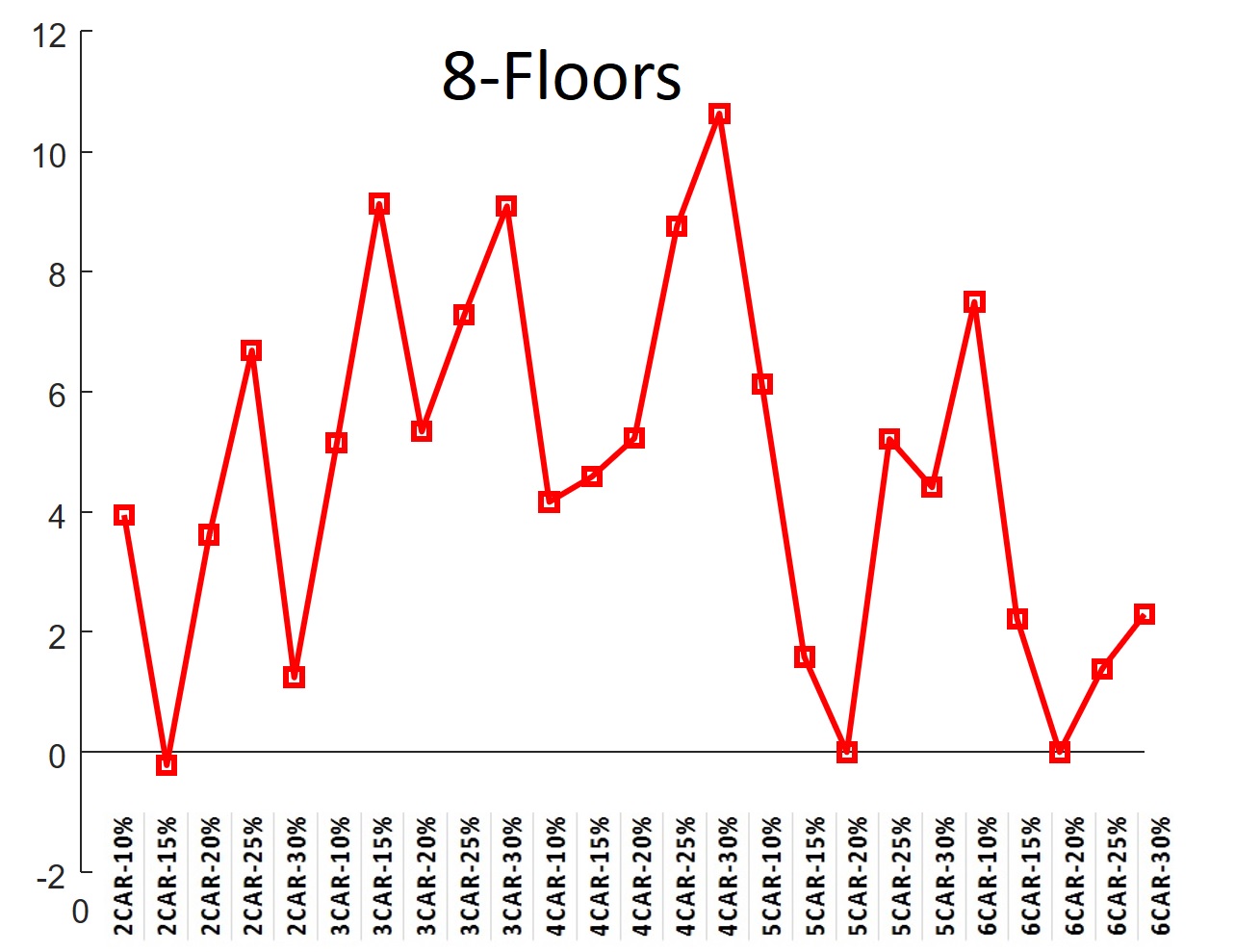}
\caption{\it We show the percentage decrease in the average waiting time using higher order submodular function 
with respect to ETA.}
\label{fg.horder_eta_verses_submodular}
\end{figure}

\section{Discussion}

We show a novel method for solving the group elevator scheduling problem by formulating it as the maximization of submodular 
functions under a matroid constraint. Our method consistently outperforms other industry-standard methods in a wide variety of 
elevator settings. 

In the future, we plan to investigate alternative methods that could directly maximize non-monotonous submodular functions~\cite{Feige2007}. This will allow us to derive improved guarantees on the optimality of the solution. The use of explicit objective function to model several design criteria also opens up the possibility of other integer programming methods for finding globally optimal solutions.  

\bibliographystyle{aaai}
\bibliography{elevSchedule}

\end{document}